\documentclass[twoside]{article}

\usepackage{PRIMEarxiv}

\usepackage[utf8]{inputenc} % allow utf-8 input
\usepackage[T1]{fontenc}    % use 8-bit T1 fonts
\usepackage{hyperref}       % hyperlinks
\usepackage{url}            % simple URL typesetting
\usepackage{booktabs}       % professional-quality tables
\usepackage{amsfonts}       % blackboard math symbols
\usepackage{nicefrac}       % compact symbols for 1/2, etc.
\usepackage{microtype}      % microtypography
\usepackage{lipsum}
\usepackage{fancyhdr}       % header
\usepackage{graphicx}       % graphics
\graphicspath{{media/}}     % organize your images and other figures under media/ folder
\usepackage{amsmath, amsthm}
\usepackage{threeparttable}

\newtheorem{proposition}{Proposition}

\title{CALIPER: Metric-Grounded Model-Free Recognition of Visually Similar Industrial Parts
}

\author{
  Alankrit Gupta, Chenxi Tao, Seung-Kyum Choi \\
  George W. Woodruff School of Mechanical Engineering \\
  Georgia Institute of Technology \\
  Atlanta, GA, U.S.A.\\
  \texttt{\{agupta3129@, ctao40@, schoi@me.\}gatech.edu} \\
}

\begin{document}
\maketitle

\begin{abstract}
Fine-grained recognition of visually similar industrial parts is challenging when classes differ primarily in physical dimensions. Normalizing detected object crops to a fixed input size suppresses absolute scale, while CAD models and large class-specific datasets may be unavailable in evolving industrial inventories. We present CALIPER, a model-free RGB-D framework that couples support-based appearance matching with metric size evidence. Each training class is onboarded from a single turntable RGB-D video and one to two labeled real images; 3D reconstruction provides novel-view appearance support, while aligned depth yields a class-specific metric size profile. At inference, a coarse YOLOv8n-seg model localizes parts, and a frozen DINOv2 backbone with an episodically trained embedding head performs fine-grained support matching. Margin-conditioned metric fusion activates probabilistic size evidence only for appearance-ambiguous decisions. New classes are enrolled from a small RGB-D support set without updating network parameters. We evaluate CALIPER on 18 visually similar industrial parts: 16 classes are used for training, while two screws are reserved for training-free enrollment. CALIPER achieves 88.2\% closed-set accuracy with 99.8\% localization recall and 85.7\% overall accuracy after 10-shot enrollment of the two unseen screws. Metric fusion improves unseen-class accuracy by up to 37.4 percentage points without statistically significant degradation of the original inventory. Robot-arm deployment identifies 17/18 parts without deployment-specific retraining.

\label{sec:00_abstract}
\end{abstract}

% keywords can be removed

\section{Introduction}
\label{sec:01_intro}

\begin{figure*}[t]
    % \vspace{1.5 mm}
    \centering
    \includegraphics[width=\linewidth]{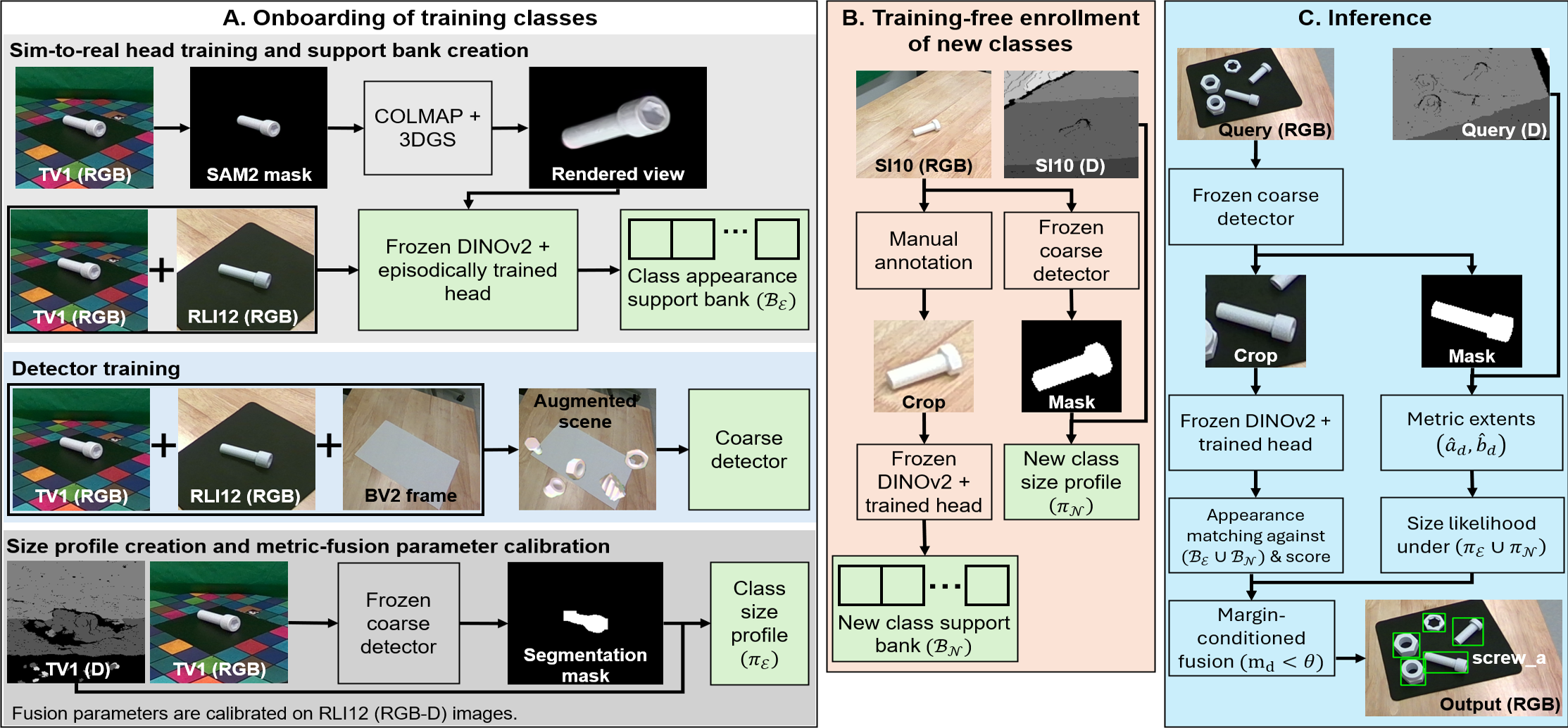}
    % \vspace{-6 mm}
    \caption{Overview of CALIPER. (A) Training classes are onboarded to build appearance support banks, a coarse detector, and class-specific metric size profiles; limited real data provide sim-to-real supervision and calibrate the fusion parameters. (B) An unseen class is enrolled from \textsc{SI10} RGB-D images by forming a new support bank and live size profile, without retraining. (C) At inference, the coarse detector provides object crops and masks for appearance matching and metric measurement, respectively. Existing and newly enrolled banks/profiles are combined, and margin-conditioned fusion produces the final fine-grained identity. Here, $\mathcal{E}$ and $\mathcal{N}$ denote existing and newly enrolled class sets; \textsc{TV1}, \textsc{RLI12}, \textsc{BV2}, \textsc{SI10} are different data tiers explained in Sec.~\ref{sec:03_dataset}.}
    \label{fig:pipeline}
    % \vspace{-5 mm}
\end{figure*}

Reliable object identification is fundamental to robotic manipulation, whether a humanoid sorts industrial components or an arm retrieves the correct instrument during robot-assisted surgery. Localization alone is insufficient if visually similar parts can still be confused. Modern detectors and fine-grained classifiers perform well when classes have distinct visual cues~\cite{NIPS2015_14bfa6bb,wei2022finegrained}, but industrial inventories often contain screws, nuts, gears, and other variants that deliberately share geometry, material, texture, and color. Distinguishing such parts can therefore depend on physical dimensions rather than appearance alone.

This setting introduces three practical challenges. First, standard detection pipelines crop and resize objects to a common input size, suppressing absolute scale and making dimensionally different parts appear similar in the visual feature space. Second, CAD models required by many model-based recognition pipelines may be unavailable or costly to maintain. Third, collecting extensive class-specific training data is expensive, particularly when an inventory evolves and new variants must be introduced after deployment. A practical recognizer should therefore operate from limited onboarding data, exploit metric information when appearance is ambiguous, and admit new classes without retraining the deployed recognition model.

Recent advances have reduced some of these requirements. Foundation-model features enable support- or template-based recognition with little task-specific training~\cite{oquab2024dinov2,nguyen2023cnos}, while model-free onboarding can replace CAD models with short reference videos~\cite{nguyen2025bopchallenge2024modelbased}. However, these approaches primarily exploit appearance and geometry for recognition, localization, or pose estimation, and are not designed for fine-grained variants whose normalized visual appearance is nearly identical. Prior work has also shown that real-world object size can aid recognition~\cite{Janoch_2011_ICCV,Xiaofei_2021_Arxiv,baghooee2026insect}, but size evidence is typically treated independently of support-based class enrollment.

This leaves an important gap for evolving industrial inventories: a new class should be enrollable from a small RGB-D support set without retraining, while carrying both an appearance representation and a metric identity profile. Moreover, metric evidence should complement rather than indiscriminately perturb appearance-based decisions, and the effect of adding new classes on the already deployed ones should be measured and bounded.

We propose CALIPER, a model-free RGB-D recognition framework that couples support-based appearance matching with metric size evidence (Fig.~\ref{fig:pipeline}). Each training class is onboarded from a single turntable RGB-D video and one to two labeled real images, with 3D reconstruction providing dense novel-view appearance support and aligned depth a class-specific metric size profile. At inference, a coarse YOLOv8n-seg model localizes candidate parts, and a frozen DINOv2 backbone with a lightweight episodically trained embedding head performs fine-grained support matching. CALIPER then applies \emph{margin-conditioned metric fusion}: depth-derived size likelihoods influence the prediction only when the pre-fusion appearance margin is small, allowing metric evidence to resolve ambiguous variants while preserving confident appearance decisions.

The same representation enables training-free enrollment. A new fine-grained class within an existing coarse category is added from a small RGB-D image set by forming its appearance support pool and live size profile, without updating network parameters. Because enrollment expands the competing label set, CALIPER also explicitly measures its interference with previously deployed classes.

We evaluate CALIPER on a controlled RGB-D benchmark of $18$ uniformly colored 3D-printed parts spanning gears, nuts, and screws. $16$ classes are used for training, while two visually similar screws are held out entirely as enrollment probes; crop normalization makes their appearance nearly identical, while their physical dimensions remain distinct.

Our contributions are:
\begin{itemize}
    \item a low-data, CAD-free onboarding pipeline that derives both appearance support and metric size profiles from RGB-D captures;
    \item margin-conditioned probabilistic fusion of appearance and metric evidence, allowing physical size to influence only appearance-ambiguous decisions;
    \item training-free enrollment of unseen fine-grained classes, with explicit measurement and a bound on appearance-channel interference with the deployed inventory;
    \item a controlled RGB-D benchmark for visually similar industrial parts, complemented by robot-arm deployment.
\end{itemize}

CALIPER reaches $88.2\%$ closed-set accuracy with $99.8\%$ localization recall. After enrolling the two unseen screws, it achieves $85.7\%$ overall accuracy with 10-shot support, with metric fusion improving the accuracy of unseen screws by up to $37.4$ percentage points (pp) without statistically significant degradation of the original inventory. In robot-arm deployment, the frozen system correctly identifies $17/18$ parts across three scenes with no retraining.
\section{Related Work}
\label{sec:02_background}

\begin{figure*}[t]
    % \vspace{2 mm}
    \centering
    \includegraphics[width=\linewidth]{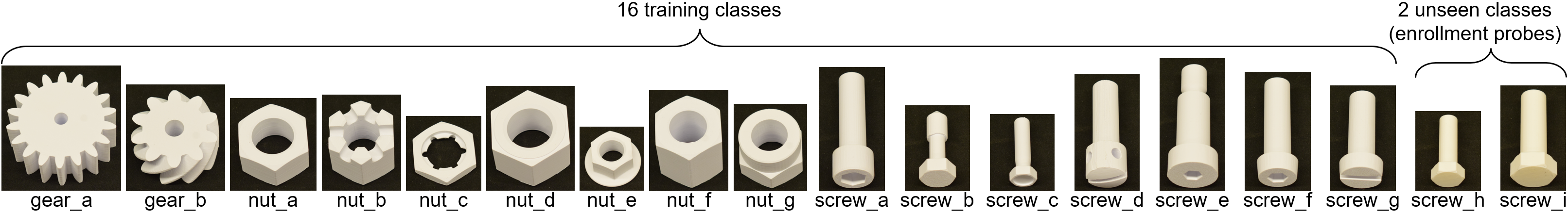}
    % \vspace{-6 mm}
    \caption{The benchmark includes 18 classes across three coarse categories -- two gears, seven nuts, and nine screws. \texttt{screw\_h} (M14) and \texttt{screw\_i} (M20) are the two unseen enrollment probes; they appear nearly identical in the visual feature space but carry distinct sizes.}
    \label{fig:obj_inventory}
    % \vspace{-6 mm}
\end{figure*}

\textbf{Foundation models and low-data recognition.}
Conventional supervised detectors and fine-grained classifiers require class-specific annotated data~\cite{NIPS2015_14bfa6bb,wei2022finegrained}. Although foundation models are pretrained at substantial data and compute scale, their representations can be reused without retraining the backbone. CLIP~\cite{radford2021clip} has enabled open-vocabulary detectors such as OWL-ViT~\cite{minderer2022owlvit} and Grounding DINO~\cite{liu2024groundingdino}, while DINOv2~\cite{oquab2024dinov2} provides transferable embeddings and SAM2~\cite{ravi2025sam} class-agnostic masks. Therefore, these models are attractive in industrial settings when labeled data are scarce and task-specific adaptation must remain lightweight.

\textbf{Few-shot classification and extensibility.}
Few-shot methods classify novel classes from limited support data. Matching Networks~\cite{vinyals2016matching} and ProtoNet~\cite{snell2017prototypical} operate on support representations, whereas MAML~\cite{finn2017maml} learns parameters for rapid adaptation; strong embeddings with simple classifiers can also be competitive~\cite{chen2019closerlook,tian2020rethinking}. Support-based classifiers can therefore add classes without modifying the feature extractor. By contrast, few-shot and incremental object detectors must also localize new classes and typically adapt detector parameters as the class set changes~\cite{wang2020tfa,qiao2021defrcn,Perez-Rua_2020_CVPR}. These works mainly study visually distinct classes rather than fine-grained identities within a coarse category, and rarely quantify interference with an existing inventory.

\textbf{Model-based and model-free object onboarding.}
CAD-based novel-object methods render templates and match them using pretrained features, as in CNOS, SAM-6D, GigaPose, and FoundPose~\cite{nguyen2023cnos,Lin_2024_CVPR,Nguyen_2024_CVPR,ornek2024foundpose}. The BOP Challenge more recently introduced model-free tasks that register unseen objects from short reference videos~\cite{nguyen2025bopchallenge2024modelbased}. OnePose~\cite{sun2022onepose} and Gen6D~\cite{liu2022gen6d} use reference images without target CAD models, while 3D reconstruction with COLMAP~\cite{schoenberger2016sfm}, NeRF~\cite{Mildenhall_2020_ECCV_NeRF}, and 3D Gaussian Splatting~\cite{KKLD23} builds dense object representations from limited captures. These methods largely target localization and pose estimation; concurrent CAD-free recognition instead fuses scale-normalized 3D shape prototypes with DINOv2 using fixed-weight fusion~\cite{tao2026appearancefailsgeometryrecognizes} but does not preserve metric scale for fine-grained identification.

\textbf{Metric size as identity evidence.}
RGB-D recognition has long used depth as a visual or geometric representation~\cite{Gupta_2014_ECCV,Eitel_2015_IROS}, usually without treating measured dimensions as explicit class evidence. Prior size-aware approaches use real-world size as semantic information~\cite{Bagherinezhad_2016_AAAI}, class-specific RGB-D size distributions for filtering and rescoring~\cite{Janoch_2011_ICCV}, size-based class exclusion~\cite{Xiaofei_2021_Arxiv}, qualitative reasoning for uncertain predictions~\cite{Chiatti_2021_AAAI_MAKE}, or probabilistic combination of size distributions with visual scores~\cite{baghooee2026insect}. However, they neither tie metric profiles to training-free support enrollment nor condition class-specific probabilistic size evidence on the appearance decision margin.

CALIPER builds on these directions through model-free video onboarding and sim-to-real support adaptation, while introducing support-derived metric profiles and margin-conditioned fusion. Metric evidence accompanies class enrollment and is invoked only for appearance-ambiguous decisions. CALIPER supports training-free enrollment and evaluates its interference with deployed classes.
\section{Dataset \& Problem Setting}
\label{sec:03_dataset}

We introduce a controlled benchmark of $18$ 3D-printed industrial parts. The 16-class training inventory contains two gears, seven nuts, and seven screws; two additional screws are reserved as \emph{unseen enrollment probes}. The M14- and M20-style screws differ in dimensions by approximately a factor of $1.4$, yet crop normalization makes their appearance nearly identical in the visual feature space. They therefore form a controlled \emph{scale-degenerate} pair for testing whether metric evidence can recover identity when appearance is ambiguous. Fig.~\ref{fig:obj_inventory} shows the complete inventory.

\begin{figure*}[t]
    % \vspace{2 mm}
    \centering
    \includegraphics[width=\linewidth]{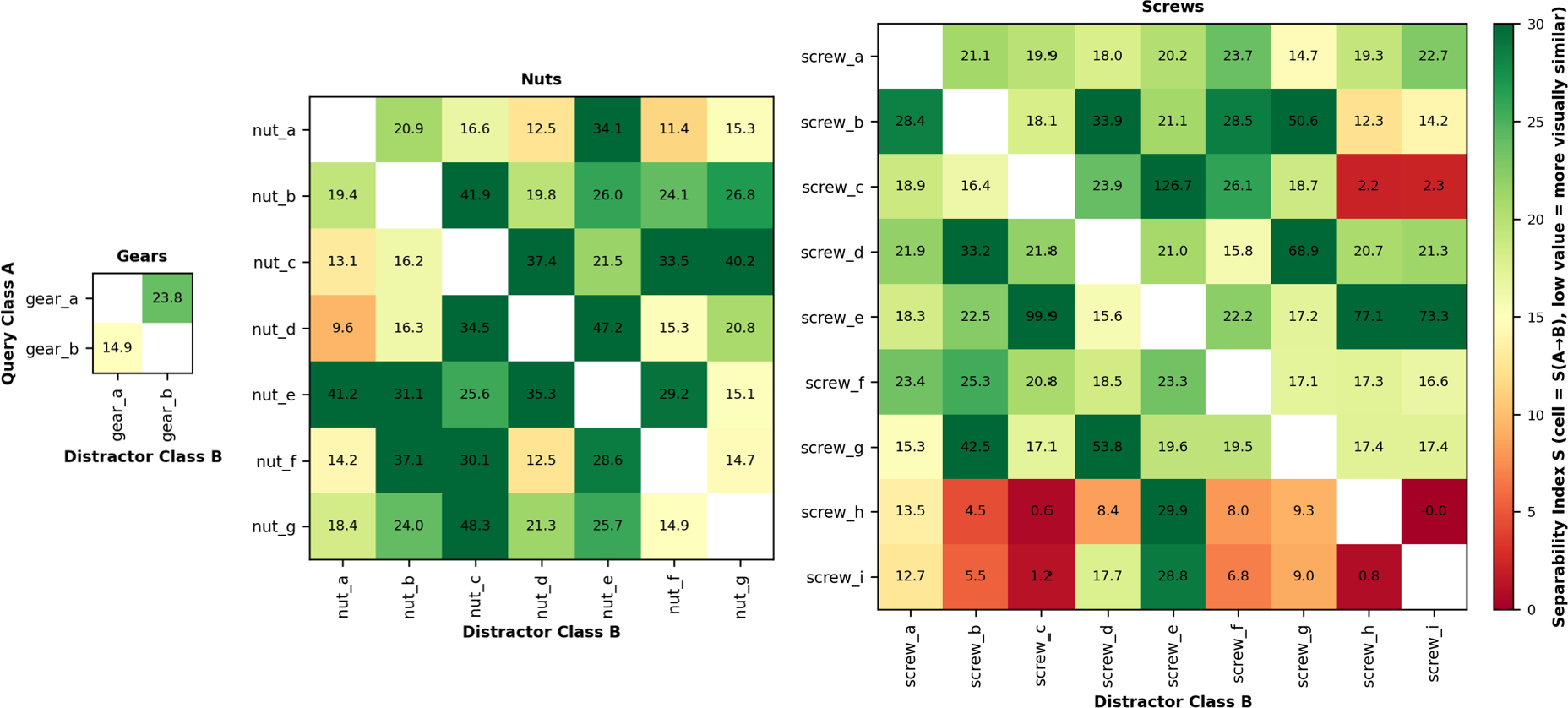}
    % \vspace{-8 mm}
    \caption{Directional separability among fine-grained classes within three coarse categories -- gears, nuts, and screws. Each cell $(A,B)$ reports $S(A{\rightarrow}B)$, measuring how well class $A$ is separated from class $B$ in the head-embedding space using onboarding ($16$ training classes) and enrollment ($2$ unseen screws) data. Lower values indicate greater visual similarity and higher confusion potential; diagonal cells are empty because self-separability is undefined.}
    \label{fig:sep_analysis}
    % \vspace{-6 mm}
\end{figure*}

\subsection{Separability Analysis}
\label{sec:03_dataset_a}

We quantify appearance difficulty using inference-time support data only, without accessing the \textsc{Test} set. Because CALIPER scores each class by maximum cosine similarity to its support pool, we define a directional separability index on the same decision variable. For a support crop $r$ of class $A$, let $s_{\mathrm{own}}(r)$ be its maximum similarity to the remaining crops of $A$ under leave-one-out evaluation, and $s_{\mathrm{other}}(r)$ its maximum similarity to class $B$. We define
% \vspace{-1 mm}
\begin{equation*}
    m(r) = s_{\mathrm{own}}(r) - s_{\mathrm{other}}(r), \qquad \quad S(A{\to}B) = \mu_m / \sigma_m,
    % \vspace{-1 mm}
\end{equation*}
where $\mu_m$ and $\sigma_m$ are the mean and standard deviation of $m(r)$ over class $A$. Thus, lower $S(A{\to}B)$ indicates greater appearance ambiguity between $A$ and $B$.

Fig.~\ref{fig:sep_analysis} reports $S$ separately for each coarse category. Most pairs are well separated, whereas the enrollment probes form the extreme case, with $S(\texttt{screw\_h}\to\texttt{screw\_i})\approx0$. A probe view is about as similar to the competing probe as to its own class. $S$ provides an \emph{a priori} appearance-difficulty order rather than a direct Test-set confusion estimate, because deployment introduces competition among all deployed classes as well as changes in pose, illumination, and occlusion.

\subsection{Data Tiers}
\label{sec:03_dataset_b}

All data are captured with an Intel RealSense D455 camera at $1280\times720$, $30$\,fps, with depth aligned to color. There are four data tiers and a disjoint evaluation set in the benchmark.

\textbf{\textsc{TV1}}. One turntable RGB-D video per training class ($\approx60$\,s, two revolutions), acquired with a static camera and a rotating object, providing appearance reconstruction and metric size profiles.

\textbf{\textsc{RLI12}}. 1--2 polygon-annotated real images per training class ($22$ total). These are the only manually labeled real images used in training, serving detector fine-tuning, sim-to-real learning, fusion calibration, and optional real support.

\textbf{\textsc{BV2}}. Two empty-workcell videos providing backgrounds for copy-paste augmentation; no annotation is required.

\textbf{\textsc{SI10}}. Ten casually captured RGB-D images per probe across varied viewpoints, distances, and illumination, each with one polygon annotation. They provide appearance support and live size profiles for enrollment.

\textbf{\textsc{Test}}. $186$ RGB-D workcell images containing $558$ annotated instances across all $18$ classes, including $30$ instances per probe. Scenes range from isolated objects to cluttered mixtures. The set is disjoint from all onboarding, training, calibration, and enrollment data.

\subsection{Task}
\label{sec:03_dataset_c}

Given a \textsc{Test} RGB-D image, the system must localize every visible part and assign its fine-grained class from the deployed inventory. We evaluate 16-class closed-set recognition and 18-class recognition after training-free enrollment of the two probes.
\section{Method}
\label{sec:04_method}

CALIPER targets fine-grained recognition of visually similar industrial parts within coarse categories such as screws, nuts, and gears, where physical dimensions provide complementary evidence to appearance. Each deployed class is therefore represented by both appearance support and a metric size profile derived from its onboarding or enrollment data. This representation also allows new classes to be added without retraining the fine-grained recognizer. Inference proceeds in two stages: a coarse instance segmenter localizes candidate objects (Sec.~\ref{sec:04_method_b}), after which each detection is identified by combining appearance matching (Sec.~\ref{sec:04_method_c}) and depth-derived metric evidence (Sec.~\ref{sec:04_method_d}) through margin-conditioned fusion (Sec.~\ref{sec:04_method_e}).

\subsection{Model-Free Onboarding of Training Classes}
\label{sec:04_method_a}

Each training class is onboarded from one \textsc{TV1} RGB-D video, from which we extract $300$ frames with aligned depth and camera intrinsics. SAM2~\cite{ravi2025sam}, prompted by a temporal-variance prior over the turntable region, provides object masks for reconstruction.

Camera poses are estimated using COLMAP~\cite{schoenberger2016sfm}, with feature extraction restricted to the rotating platform, and a 3D Gaussian Splatting model~\cite{KKLD23} is optimized from the masked sequence. We render $360$ novel views spanning three elevation bands and $120$ azimuths per band. Rendered views and real images undergo the same preprocessing (masking, background normalization, cropping, and resizing) before feature extraction. The appearance bank for class $c$ is $\mathcal{B}_c=\{g_\phi(f(v)):v \in \mathcal{V}_c\},$ where $\mathcal{V}_c$ denotes its rendered and real views, $f$ is the frozen DINOv2 backbone and $g_\phi$ is the learned embedding head.

The same RGB-D sequence also provides physical-size measurements. For each frame, detector-masked depth pixels are back-projected using the camera intrinsics under the pinhole camera geometry, and principal component analysis of the resulting planar point set yields major and minor extents $(a_i,b_i)$. The extents are modeled in log space, giving the class size profile as $\pi_c=\left(\mu_a^c,\sigma_a^c,\mu_b^c,\sigma_b^c\right)$, where each $\mu$ and $\sigma$ is the sample mean and standard deviation of the corresponding log-extent. In addition to \textsc{TV1}, \textsc{RLI12} provides limited real-domain supervision.

\subsection{Coarse Object Localization}
\label{sec:04_method_b}

CALIPER separates object localization from fine-grained identification. A lightweight YOLOv8n-seg model~\cite{yolov8_ultralytics} is fine-tuned only for coarse categories screw, nut, and gear, allowing an unseen fine-grained class to be localized under an existing coarse category without retraining the detector. We extract about $130$ background frames from \textsc{BV2} and use copy-paste augmentation~\cite{Ghiasi_2021_CVPR} to place segmented \textsc{TV1} and \textsc{RLI12} objects onto these real backgrounds. Each detection $d$ provides a crop $x_d$ for appearance matching and a segmentation mask $M_d$ for metric measurement.

\subsection{Head Training and Appearance Scoring}
\label{sec:04_method_c}

Each object crop is encoded by the CLS token of a frozen DINOv2 ViT-B/14 backbone. A lightweight two-layer MLP head $g_\phi:\mathbb{R}^{768} \rightarrow \mathbb{R}^{512} \rightarrow \mathbb{R}^{256}$ maps the backbone feature to an L2-normalized embedding. It is the only trainable component of the fine-grained recognition stage.

The head is trained episodically in a few-shot meta-learning framework to bridge the domain gap between reconstructed views and real workcell images. In each episode, support samples are drawn from rendered views, while queries are drawn from \textsc{RLI12} and $25$ real frames sampled per class from \textsc{TV1}. A soft nearest-neighbor objective pulls real queries toward same-class rendered supports while separating competing classes. These \emph{sim-to-real episodes} reproduce the deployment setting in which rendered class supports are matched against real detections. The trained head is used at the same support budget $K$ as used in training.

At inference, class $c$ is represented by a support pool $\mathcal{P}_c$ drawn from its onboarding or enrollment data according to the support budget. Its appearance score for a detection $d$ is
% \vspace{-2 mm}
\begin{equation*}
    s_c(d) = \max_{p\in\mathcal{P}_c} \left\langle g_\phi(f(x_d)),p\right\rangle .
    % \vspace{-1 mm}
\end{equation*}
All embeddings are L2-normalized, so the inner product is cosine similarity. Let $m_d$ denote the \emph{appearance margin} as the absolute difference between the two largest class scores; smaller values indicate greater appearance ambiguity.

\subsection{Metric Measurement and Size Likelihood}
\label{sec:04_method_d}

For each detection, mask $M_d$ is back-projected through the aligned depth using the same geometric measurement procedure used to construct the size profiles, yielding major and minor extents $(\hat a_d,\hat b_d)$. Using a consistent segmentation and measurement convention for profiles and queries prevents systematic mask or depth biases from being interpreted as class-size differences.

Under class profile $\pi_c$, the metric log-likelihood is
% \vspace{-1 mm}
\begin{equation*}
\ell_c(d) = -\frac{1}{2}\left(\frac{\log\hat a_d-\mu_a^c}{\sigma_a^c}\right)^2 - \log\sigma_a^c - \frac{1}{2}\left(\frac{\log\hat b_d - \mu_b^c}{\sigma_b^c}\right)^2 - \log\sigma_b^c
% \vspace{-1 mm}
\end{equation*}

The two extents are modeled independently in log space, so uncertainty reflects relative size variation rather than absolute deviations, while ensuring positive physical dimensions. Each standard deviation is lower-bounded by $0.06$ to avoid unrealistically sharp class distributions.

Normalizing across the deployed class set $\mathcal{C}$ gives
% \vspace{-1 mm}
\begin{equation*}
    L_c(d) = \ell_c(d) - \log \sum_{c' \in \mathcal{C}} \exp\left(\ell_{c'}(d)\right),
    % \vspace{-1 mm}
\end{equation*}
which is the class log-posterior under a uniform prior. If metric measurement is unreliable, $L_c(d)$ is made constant across classes and therefore does not affect the prediction.

\subsection{Margin-Conditioned Metric Fusion}
\label{sec:04_method_e}

Metric evidence is activated only when the pre-fusion appearance margin $m_d$ falls below $\theta$. The final score is
% \vspace{-1 mm}
\begin{equation*}
    \mathrm{score}_c(d) = s_c(d) + \lambda \left(\mathbb{1} [m_d<\theta]\right) L_c(d),
    % \vspace{-1 mm}
\end{equation*}
Thus, metric fusion can resolve visually ambiguous classes with distinct sizes, while confident appearance predictions are left unchanged. The fusion weight $\lambda$ and ambiguity threshold $\theta$ are calibrated once using the real labeled images from \textsc{RLI12} and then frozen for all test evaluations. We use $\theta=0.12$ and $\lambda=0.0173$ throughout.

\subsection{Training-Free Enrollment of Unseen Classes}
\label{sec:04_method_f}

CALIPER enrolls an unseen fine-grained class without retraining, provided it belongs to a coarse category already covered by the detector. A small set of annotated crops from \textsc{SI10} spanning multiple viewpoints is identically preprocessed and embedded by the frozen backbone and trained head to form the new support pool.

The same enrollment images also define a \emph{live size profile}. Their detector masks and aligned depth are processed using the inference-time metric measurement pipeline, and the resulting physical extents are used to estimate $\pi_c$. The new class then participates directly in the same $|\mathcal{C}|$-way appearance scoring and margin-conditioned fusion as the existing inventory, with no class-specific optimization.
\section{Experiments And Discussion}
\label{sec:05_experiments}

\subsection{Experimental Setup and Baselines}
\label{sec:05_experimental_setup}

We evaluate CALIPER and competing methods on the reserved \textsc{Test} set. CALIPER first localizes objects using the coarse YOLOv8n-seg model of Sec.~\ref{sec:04_method_b}, followed by fine-grained recognition. Detections are greedily matched in decreasing confidence order to at most one unmatched ground-truth instance by class-agnostic matching at IoU $\geq 0.5$. The detector localizes $557/558$ instances, yielding $99.8\%$ localization recall.

Fine-grained \emph{accuracy} is computed only over matched detections, separating identification from localization. In the 18-class setting, we also report accuracy on the 16 original classes (\textit{Training16}) and the two enrolled unseen classes (\textit{Unseen2}). Support-dependent results are averaged over $100$ random draws for $K\in\{1,5,10\}$.

For selected pairwise comparisons, we use McNemar's exact test~\cite{McNemar_1947,dietterich1998} on the same matched detections and support draw. Let $b$ count detections correctly classified by CALIPER but not the compared method, and $c$ the reverse. Under the null hypothesis that neither method is more likely to be correct on their disagreements, the test returns a two-sided $p$-value. Smaller $p$ provides stronger evidence of a systematic accuracy difference; $p<0.05$ is treated as statistically significant. Reported $b$, $c$, and $p$ use the first support draw, and the same qualitative conclusions hold over the first ten draws.

We compare against four baselines. \emph{One-stage YOLO} fine-tunes a pretrained YOLOv8n-seg model directly on fine-grained identities and must be retrained after new-class enrollment. The remaining baselines share CALIPER's coarse detector and frozen DINOv2 backbone. \emph{Raw DINOv2} performs CNOS-style~\cite{nguyen2023cnos} max-over-support matching in the frozen feature space. \emph{Linear classifier} trains multinomial logistic regression on frozen DINOv2 features and requires retraining after new class enrollment. It is evaluated over five random $K=1$ draws and two $K=5$ partitions, while $K=10$ uses all enrollment images. \emph{ProtoNet} uses episodic training with mean-prototype classification as a standard few-shot alternative to CALIPER's max-over-support scoring. Unless stated otherwise, the baselines do not use metric fusion.

\begin{table}[!b]
% \vspace{-5 mm}
\centering
\caption{Closed-set accuracy (\%) on the 16 training classes. Support-based results are averaged over $100$ random draws. The default CALIPER configuration is marked in bold font.}
% \vspace{-2 mm}
\label{tab:closedset}
\small
\begin{threeparttable}
\begin{tabular*}{0.51\textwidth}{@{\extracolsep{\fill}} lccc}
\toprule
Method & $K=1$ & $K=5$ & $K=10$ \\
\midrule
One-stage YOLO\tnote{*}        & 63.4 & 63.4 & 63.4 \\
Raw DINOv2                     & 39.7 & 71.9 & 75.8 \\
Linear classifier\tnote{*}          & 85.3 & 85.3 & 85.3 \\
ProtoNet                       & 87.3 & 88.3 & 87.6 \\
CALIPER, rendered support      & 87.9 & 87.9 & 88.2 \\
CALIPER, real support          & 88.2 & 87.4 & 88.2 \\
\textbf{CALIPER, rendered+real}
                               & \textbf{88.1} & \textbf{87.8} & \textbf{88.2} \\
\bottomrule
\end{tabular*}
\begin{tablenotes}[flushleft]\footnotesize
\item[*] Trained once on all available training data; these methods do not use a $K$-shot support pool, so their accuracy does not depend on $K$.
\end{tablenotes}
\end{threeparttable}
\end{table}

\begin{table}[t]
    % \vspace{1.5 mm}
    \centering
    \caption{Accuracy (\%) after enrolling two unseen classes. Support-based results are averaged over $100$ random support draws. For methods sharing CALIPER's coarse detector, \emph{Overall}, \emph{Training16}, and \emph{Unseen2} contain $557$, $497$, and $60$ matched detections, respectively. \emph{One-stage YOLO} is evaluated on detections from its retrained fine-grained detector, trained once per $K$.}
    \label{tab:enrollment}
    % \vspace{-3 mm}
    \small
    \begin{tabular}{l ccc ccc ccc}
    
    \toprule
    & \multicolumn{3}{c}{Overall}
    & \multicolumn{3}{c}{Training16}
    & \multicolumn{3}{c}{Unseen2} \\
    \cmidrule(lr){2-4}
    \cmidrule(lr){5-7}
    \cmidrule(lr){8-10}
    
    Method
    & $K=1$ & $K=5$ & $K=10$    & $K=1$ & $K=5$ & $K=10$   & $K=1$ & $K=5$ & $K=10$ \\
    \midrule
    
    One-stage YOLO
    & 61.4 & 53.8 & 55.1        & 63.6 & 53.4 & 53.4        & 42.9 & 56.9 & 68.4 \\    
    Raw DINOv2
    & 37.2 & 64.1 & 67.5        & 36.3 & 66.2 & 69.7        & 44.2 & 47.4 & 49.4 \\    
    Linear classifier
    & 71.5 & 78.6 & 78.8        & 77.6 & 83.0 & 81.1        & 20.3 & 41.7 & 60.0 \\    
    ProtoNet
    & 79.1 & 80.8 & 81.1        & 85.3 & 87.0 & 86.3        & 27.7 & 29.8 & 38.0 \\    
    % CALIPER, rendered support
    % & 83.4 & 85.5 & 85.6        & 87.1 & 86.8 & 86.6        & 52.9 & 75.1 & 77.3 \\    
    \textbf{CALIPER}
    & \textbf{83.3} & \textbf{85.1} & \textbf{85.7} & \textbf{87.3} & \textbf{86.6} & \textbf{86.9} & \textbf{50.0} & \textbf{72.5} & \textbf{75.3} \\
    
    \bottomrule
    \end{tabular}
    % \vspace{-6 mm}
\end{table}

\subsection{Closed-Set Recognition}
\label{sec:closed_set}

Table~\ref{tab:closedset} evaluates recognition on the 16 training classes. \emph{One-stage YOLO} reaches only $63.4\%$, showing that limited data are insufficient for directly fine-tuning the detector to fine-grained identities. \emph{Raw DINOv2} improves with additional support but remains substantially weaker, reaching $75.8\%$ at $K=10$. The \emph{Linear classifier} performs considerably better, reaching $85.3\%$. CALIPER significantly outperforms the \emph{Linear classifier} at all three support budgets (McNemar, $p<0.05$ for each) and performs comparably to \emph{ProtoNet}. Across CALIPER variants, rendered-only, real-only, and mixed support differ by at most $0.5$ pp, showing that the sim-to-real head makes rendered onboarding views effective support for real detections.

\subsection{Enrollment of Unseen Classes}
\label{sec:enrollment}

We next enroll \texttt{screw\_h} and \texttt{screw\_i}, which are absent from both detector and embedding-head training. Their appearance support and live size profiles are obtained solely from \textsc{SI10}, while the Training16 support pools remain unchanged. The results are reported in Table~\ref{tab:enrollment}.

CALIPER achieves the highest overall accuracy at every support budget, reaching $85.7\%$ at $K=10$ compared with $81.1\%$ for \emph{ProtoNet}. The advantage is concentrated on Unseen2: CALIPER reaches $72.5\%$ and $75.3\%$ on Unseen2 at $K=5$ and $10$, while \emph{ProtoNet} reaches $29.8\%$ and $38.0\%$, respectively. Their overall difference is statistically significant at every $K$, as is the Unseen2 difference for $K\geq5$, consistent with the small $p$-values in Table~\ref{tab:mcnemar_enroll}. The higher Training16 $p$-values provide no evidence of a difference.

\begin{table}[!b]
    % \vspace{-4 mm}
    \centering
    \caption{McNemar's exact test ($b$, $c$, $p$) for CALIPER against \emph{ProtoNet} and the \emph{Linear classifier} in the 18-class setting of Table~\ref{tab:enrollment}.}
    % \vspace{-3 mm}
    \label{tab:mcnemar_enroll}

    \small

    \begin{tabular}{l ccc ccc ccc}
    \toprule

    & \multicolumn{3}{c}{$K=1$}
    & \multicolumn{3}{c}{$K=5$}
    & \multicolumn{3}{c}{$K=10$} \\
    \cmidrule(lr){2-4}
    \cmidrule(lr){5-7}
    \cmidrule(lr){8-10}

    Comparison
    & $b$ & $c$ & $p$           & $b$ & $c$ & $p$           & $b$ & $c$ & $p$ \\
    \midrule

    \multicolumn{10}{l}{CALIPER vs Linear classifier} \\
    \quad Overall
    & $94$ & $24$ & $<0.001$    & $64$ & $23$ & $<0.001$    & $62$ & $22$ & $<0.001$ \\
    \quad Training16
    & $73$ & $19$ & $<0.001$    & $40$ & $16$ & $0.002$      & $45$ & $14$ & $<0.001$ \\
    \quad Unseen2
    & $21$ & $5$ & $0.002$      & $24$ & $7$ & $0.003$      & $17$ & $8$ & $0.108$ \\
    \midrule
    
    \multicolumn{10}{l}{CALIPER vs ProtoNet} \\
    \quad Overall
    & $22$ & $8$ & $0.016$      & $47$ & $13$ & $<0.001$    & $42$ & $14$ & $<0.001$ \\
    \quad Training16
    & $13$ & $5$ & $0.096$      & $12$ & $11$ & $1.00$      & $15$ & $10$ & $0.424$ \\
    \quad Unseen2
    & $9$ & $3$ & $0.146$       & $35$ & $2$ & $<0.001$     & $27$ & $4$ & $<0.001$ \\
    \bottomrule
    \end{tabular}
\end{table}

CALIPER also substantially outperforms \emph{Raw DINOv2} across Overall, Training16, and Unseen2. Against the \emph{Linear classifier}, CALIPER is significantly better overall at every $K$ ($p < 0.001$) and on Unseen2 at $K=1$ and $K=5$ ($p \le 0.003$).

Retraining is less stable. At $K=10$, \emph{One-stage YOLO} reaches $68.4\%$ on Unseen2 but only $53.4\%$ on Training16, and $94.1\%$ localization recall. CALIPER instead adds the unseen classes without retraining any network parameters.

\begin{table}[htbp]
% \vspace{1.5 mm}
\centering
\caption{Ablation of margin-conditioned metric fusion. Accuracy (\%) is reported with and without metric fusion for the 16-class closed-set and 18-class enrollment settings, averaged over $100$ random support draws; $\Delta$ denotes the change due to fusion.}
% \vspace{-3 mm}
\label{tab:gate}
\small
\setlength{\tabcolsep}{3pt}
\begin{tabular}{@{}l ccc ccc ccc ccc@{}}
\toprule
& \multicolumn{3}{c}{16-class Training16}
& \multicolumn{3}{c}{18-class Overall}
& \multicolumn{3}{c}{18-class Training16}
& \multicolumn{3}{c}{18-class Unseen2} \\
\cmidrule(lr){2-4}
\cmidrule(lr){5-7}
\cmidrule(lr){8-10}
\cmidrule(lr){11-13}
Method
& $K=1$ & $K=5$ & $K=10$
& $K=1$ & $K=5$ & $K=10$
& $K=1$ & $K=5$ & $K=10$
& $K=1$ & $K=5$ & $K=10$ \\
\midrule
No metric fusion
& 87.1 & 87.3 & 88.3
& 78.5 & 80.4 & 81.2
& 84.8 & 85.5 & 86.5
& 26.8 & 38.3 & 37.9 \\

With metric fusion
& 88.1 & 87.8 & 88.2
& 83.3 & 85.1 & 85.7
& 87.3 & 86.6 & 86.9
& 50.0 & 72.5 & 75.3 \\

$\Delta$
& +1.0 & +0.5 & $-0.1$
& +4.8 & +4.7 & +4.5
& +2.5 & +1.1 & +0.4
& +23.2 & +34.2 & +37.4 \\
\bottomrule
\end{tabular}
% \vspace{-7 mm}
\end{table}

We next quantify the cost of expanding the deployed inventory. We score the same $497$ detections before (16-way, Table~\ref{tab:closedset}) and after enrolling the two probes (18-way, Table~\ref{tab:enrollment}), so the only change is competition from the new classes. As shown in Fig.~\ref{fig:enrollment_cost}, Training16 accuracy decreases by only $0.8$, $1.2$, and $1.3$ pp at $K=1$, $5$, and $10$, while Unseen2 accuracy increases from $50.0\%$ at $K=1$ to $75.3\%$ at $K=10$.

The decrease arises due to \emph{interference}, the number of Training16 detections for which a newly enrolled class outscores every existing class in the appearance-channel. This occurs for an average of $13.9$, $10.6$, and $13.1$ of the $497$ Training16 detections at $K\in\{1,5,10\}$, respectively. Margin-conditioned fusion returns part of these to their true class, leaving only $5.5$, $7.8$, and $9.3$ detections assigned to a probe. Because some of these detections were already misclassified before enrollment, the resulting accuracy loss is smaller than the capture count. For a fixed embedding head and existing-class support pools, the complete enrollment banks provide an upper bound on this interference, as formalized below.

\begin{figure}[!b]
    % \vspace{-4 mm}
    \centering
    \includegraphics[width=0.52\textwidth]{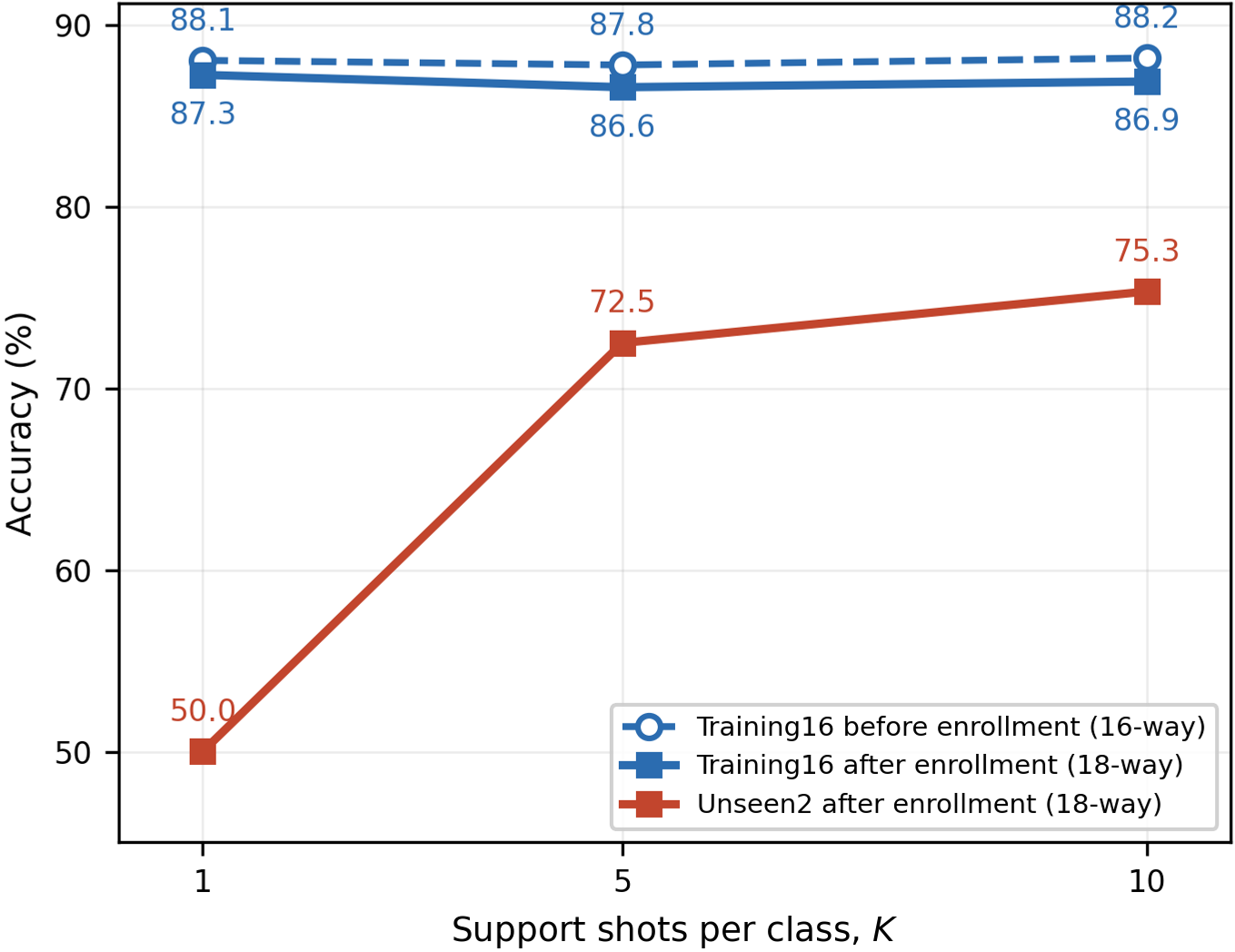}
    % \vspace{-7 mm}
    \caption{Effect of training-free enrollment on existing and unseen classes. Accuracies are averaged over $100$ random support draws for each $K$. \emph{Training16} (blue) is evaluated before enrollment (16-way, dashed) and after adding the two probes (18-way, solid). \emph{Unseen2} (red) reports enrolled-class accuracy in the 18-way setting.}
    \label{fig:enrollment_cost}
\end{figure}

Let $\mathcal{E}$ and $\mathcal{N}$ denote the existing and newly enrolled class sets, respectively, and let $\mathcal{D}_{\mathcal{E}}$ contain detections belonging to $\mathcal{E}$. Writing $s_c(d; \mathcal{P}_c)$ for the appearance score of Sec.~\ref{sec:04_method_c} using support pool $\mathcal{P}_c$, define the capture set
% \vspace{-1 mm}
\begin{equation*}
    C(\mathcal{P}_\mathcal{N})=\Bigl\{d\in\mathcal{D}_{\mathcal{E}} : \max_{n\in\mathcal{N}} s_n(d;\mathcal{P}_n) > \max_{e\in\mathcal{E}} s_e(d;\mathcal{P}_e)\Bigr\},
    % \vspace{-1 mm}
\end{equation*}
with appearance-channel interference $I(\mathcal{P}_\mathcal{N})=|C(\mathcal{P}_\mathcal{N})|$.

\begin{proposition}[Bounded interference]
    \label{prop:1}
    Fix the embedding and the existing-class support pools. If $\mathcal{P}_n \subseteq \mathcal{B}_n,\ \forall n\in\mathcal{N}$, then $C(\mathcal{P}_\mathcal{N})\subseteq C(\mathcal{B}_\mathcal{N})$ and $I(\mathcal{P}_\mathcal{N}) \leq I(\mathcal{B}_\mathcal{N}),$ where $\mathcal{B}_n$ is class $n$'s complete enrollment bank and $\mathcal{B}_\mathcal{N}=\{\mathcal{B}_n\}_{n\in\mathcal{N}}$.
\end{proposition}

\begin{proof}
    For $d\in C(\mathcal{P}_\mathcal{N})$, there is an $n$ such that $s_n(d;\mathcal{P}_n) > \max_{e\in\mathcal{E}} s_e(d;\mathcal{P}_e)$. Since $\mathcal{P}_n\subseteq\mathcal{B}_n$, $s_n(d;\mathcal{P}_n)\leq s_n(d;\mathcal{B}_n)$ for all $n \in \mathcal{N}$, so $d\in C(\mathcal{B}_\mathcal{N})$. Thus, $C(\mathcal{P}_\mathcal{N})\subseteq C(\mathcal{B}_\mathcal{N})$. Taking cardinalities gives $I(\mathcal{P}_\mathcal{N}) \leq I(\mathcal{B}_\mathcal{N})$; with equality when $\mathcal{P}_n=\mathcal{B}_n$ for all $n\in\mathcal{N}$.
\end{proof}

With the complete \textsc{SI10} banks, $I(\mathcal{B}_\mathcal{N})$ is $32.7$, $14.0$, and $13.1$ captures at $K\in\{1,5,10\}$, against the observed $13.9$, $10.6$, and $13.1$, with the bound attained at $K=10$.

\subsection{Effect of Metric Fusion}
\label{sec:metric_fusion_exp}

Table~\ref{tab:gate} isolates the effect of margin-conditioned metric fusion. In the 16-class closed set, fusion changes accuracy by only $-0.1$ to $+1.0$ pp. After enrollment, it improves Training16 by $+0.4$ to $+2.5$ pp, while producing much larger gains of $+23.2$ to $+37.4$ pp on Unseen2 and $+4.5$ to $+4.8$ pp overall. The overall gain is statistically significant at every $K$ (McNemar $p\leq0.004$), and the Unseen2 gain is significant for $K\geq5$ (McNemar $p<0.001$); no statistically significant degradation is observed on the original 16 classes. Thus, margin-conditioned fusion concentrates its benefit on the size-distinct ambiguous probes while preserving performance on the original inventory.

Table~\ref{tab:probeconf} shows the mechanism at $K=10$. Without metric fusion, the two probes are confused with one another almost as often as they are correctly identified. Fusion eliminates the \texttt{screw\_h}$\rightarrow$\texttt{screw\_i} confusion and substantially reduces the reverse confusion. This confirms that metric size can recover identity information suppressed by crop normalization.

The observed confusions also agree with the separability analysis of Sec.~\ref{sec:03_dataset_a}, which uses no \textsc{Test} data. Among the $116$ ordered within-category class pairs, all $18$ directions with at least $5\%$ confusion have $S<20$, whereas pairs with $S\geq20$ account for only $3\%$ of the total off-diagonal confusion. Fusion resolves the probe-pair ambiguity because their metric sizes are distinct; residual errors occur mainly where both appearance and measured-size distributions overlap.

\begin{figure*}[t]
    % \vspace{1.5 mm}
    \centering
    \includegraphics[width=\linewidth]{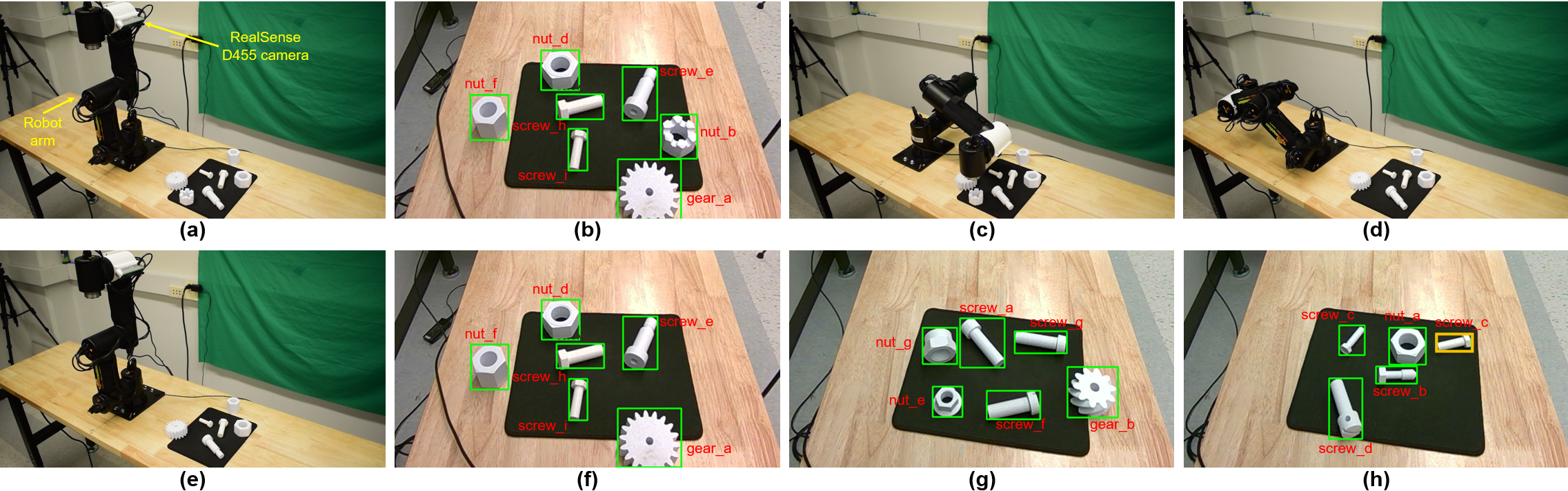}
    % \vspace{-8 mm}
    \caption{Robot-arm deployment. (a,b) A seven-part scene and classification. (c,d) Motion above \texttt{nut\_b} and to the predefined drop pose; no gripper is used. (e,f) Reclassification after one object is removed. (g,h) Additional scenes; (h) shows the only error (orange bounding box), \texttt{screw\_h}$\rightarrow$\texttt{screw\_c}. All panels use frozen $K=10$ with no deployment-specific retraining or recalibration.}
    \label{fig:deployment}
    % \vspace{-6 mm}
\end{figure*}

\subsection{Deployment on a Robot Arm}
\label{sec:deployment}

We deploy CALIPER on a $6$-DoF robot arm with a wrist-mounted RealSense D455 camera in a closed perception-motion loop. At each cycle, the arm returns to a home pose, CALIPER classifies the scene, and a selected detection is back-projected into the arm frame. The arm then moves above the target and to a predefined drop pose before returning home. No gripper is used; the experiment therefore evaluates recognition and perception-to-motion integration rather than grasping performance.

Deployment uses the same frozen $K=10$ configuration as Table~\ref{tab:enrollment}, with both probes enrolled solely from \textsc{SI10}, and the fusion parameters of Sec.~\ref{sec:04_method_e}. No model parameters, thresholds, or size profiles are recalibrated for the robot.

Fig.~\ref{fig:deployment} (b, g, h) shows three scene compositions with seven, six, and five parts, so that each of the $18$ inventory classes appears exactly once; CALIPER correctly identifies $17/18$. All seven parts in the initial scene, including both enrollment probes, are correctly recognized, and the remaining six are again correct after one object is removed. The only error is \texttt{screw\_h}$\rightarrow$\texttt{screw\_c} in Fig.~\ref{fig:deployment} (h), consistent with the residual confusion in Table~\ref{tab:probeconf}, where the two classes remain ambiguous in both appearance and size.

\begin{table}[h]
\centering
% \vspace{-2 mm}
\caption{Row-normalized confusion (\%) for the two enrolled classes at $K=10$ in the 18-class setting, with and without fusion.}
% \vspace{-3 mm}
\label{tab:probeconf}
\small
\setlength{\tabcolsep}{3.5pt}
\begin{tabular}{lcccc}
\toprule
True class
& \texttt{screw\_h}
& \texttt{screw\_i}
& \texttt{screw\_c}
& Other \\
\midrule
\texttt{screw\_h}, no fusion
& 42.5 & 31.3 & 22.9 & 3.3 \\

\texttt{screw\_i}, no fusion
& 39.9 & 33.3 & 10.9 & 15.9 \\

\texttt{screw\_h}, fused
& \textbf{82.0} & \textbf{0.0} & 14.7 & 3.3 \\

\texttt{screw\_i}, fused
& 16.7 & \textbf{68.7} & 0.0 & 14.6 \\
\bottomrule
\end{tabular}
% \vspace{-4 mm}
\end{table}

\section{Conclusion And Future Work}
\label{sec:06_conclusion}

CALIPER addresses evolving inventories where visually similar parts must be distinguished from limited data without CAD models. By coupling appearance support with metric size profiles, new classes can be enrolled without retraining while physical size is used only when appearance is ambiguous. This is relevant to industrial sorting, assembly, and retrieval where similar variants are not interchangeable.

The current formulation requires new classes to belong to coarse categories already covered by the detector and remains limited when classes overlap in both appearance and size. Future work will target new coarse categories, richer geometric cues, and larger real industrial inventories.

\bibliographystyle{unsrt}  
\bibliography{references}

@inproceedings{NIPS2015_14bfa6bb,
 author = {Ren, Shaoqing and He, Kaiming and Girshick, Ross and Sun, Jian},
 booktitle = {Advances in Neural Information Processing Systems},
 editor = {},
 pages = {},
 publisher = {},
 title = {{Faster R-CNN: Towards Real-Time Object Detection with Region Proposal Networks}},
 volume = {28},
 year = {2015}
}

@misc{yolov8_ultralytics,
  author       = {Glenn Jocher and Ayush Chaurasia and Jing Qiu},
  title        = {Ultralytics {YOLOv8}},
  year         = {2023},
  howpublished = {\url{https://github.com/ultralytics/ultralytics}},
  note         = {{A}ccessed: Sep. 2, 2026}
}

@inproceedings{ravi2025sam,
title={{SAM 2: Segment Anything in Images and Videos}},
author={Nikhila Ravi and Valentin Gabeur and Yuan-Ting Hu and Ronghang Hu and Chaitanya Ryali and Tengyu Ma and Haitham Khedr and Roman R{\"a}dle and Chloe Rolland and Laura Gustafson and Eric Mintun and Junting Pan and Kalyan Vasudev Alwala and Nicolas Carion and Chao-Yuan Wu and Ross Girshick and Piotr Dollar and Christoph Feichtenhofer},
booktitle={The Thirteenth International Conference on Learning Representations},
year={2025},
}

@inproceedings{nguyen2023cnos,
    title     = {{CNOS: A Strong Baseline for CAD-based Novel Object Segmentation}},
    author    = {Nguyen, Van Nguyen and Groueix, Thibault and Ponimatkin, Georgy and Lepetit, Vincent and Hodan, Tomas},
    booktitle = {Proceedings of the IEEE/CVF ICCV},
    pages     = {2134--2140},
    year      = {2023}
  }

@Article{KKLD23,
  author       = "Kerbl, Bernhard and Kopanas, Georgios and Leimkühler, Thomas and Drettakis, George",
  title        = "3D Gaussian Splatting for Real-Time Radiance Field Rendering",
  journal      = "ACM Transactions on Graphics (SIGGRAPH Conference Proceedings)",
  number       = "",
  volume       = "42",
  month        = "July",
  year         = "2023",
}

@misc{nguyen2025bopchallenge2024modelbased,
      title={{BOP Challenge 2024 on Model-Based and Model-Free 6D Object Pose Estimation}}, 
      author={Van Nguyen Nguyen and Stephen Tyree and Andrew Guo and Mederic Fourmy and Anas Gouda and Taeyeop Lee and Sungphill Moon and Hyeontae Son and Lukas Ranftl and Jonathan Tremblay and Eric Brachmann and Bertram Drost and Vincent Lepetit and Carsten Rother and Stan Birchfield and Jiri Matas and Yann Labbe and Martin Sundermeyer and Tomas Hodan},
      year={2025},
      eprint={2504.02812},
      archivePrefix={arXiv},
      primaryClass={cs.CV},
}

@inproceedings{schoenberger2016sfm,
    author={Sch\"{o}nberger, Johannes Lutz and Frahm, Jan-Michael},
    title={{Structure-from-Motion Revisited}},
    booktitle={Conference on Computer Vision and Pattern Recognition2016},
}

@InProceedings{Ghiasi_2021_CVPR,
    author    = {Ghiasi, Golnaz and Cui, Yin and Srinivas, Aravind and Qian, Rui and Lin, Tsung-Yi and Cubuk, Ekin D. and Le, Quoc V. and Zoph, Barret},
    title     = {{Simple Copy-Paste Is a Strong Data Augmentation Method for Instance Segmentation}},
    booktitle = {Proceedings of the IEEE/CVF Conference on Computer Vision and Pattern Recognition},
    year      = {2021},
}

@ARTICLE{wei2022finegrained,
author={Wei, Xiu-Shen and Song, Yi-Zhe and Aodha, Oisin Mac and Wu, Jianxin and Peng, Yuxin and Tang, Jinhui and Yang, Jian and Belongie, Serge},
journal={ IEEE Transactions on Pattern Analysis \& Machine Intelligence },
title={{ Fine-Grained Image Analysis With Deep Learning: A Survey }},
year={2022},
volume={44},
number={12},
ISSN={1939-3539},
pages={8927-8948},
doi={10.1109/TPAMI.2021.3126648},
publisher={IEEE Computer Society},
address={Los Alamitos, CA, USA},
month=dec}

@article{radford2021clip,
  title={{Learning Transferable Visual Models From Natural Language Supervision}},
  author={Alec Radford and Jong Wook Kim and Chris Hallacy and Aditya Ramesh and Gabriel Goh and Sandhini Agarwal and Girish Sastry and Amanda Askell and Pamela Mishkin and Jack Clark and Gretchen Krueger and Ilya Sutskever},
  journal={ArXiv},
  year={2021},
  volume={abs/2103.00020},
}

@inproceedings{minderer2022owlvit,
author = {Minderer, Matthias and Gritsenko, Alexey and Stone, Austin and Neumann, Maxim and Weissenborn, Dirk and Dosovitskiy, Alexey and Mahendran, Aravindh and Arnab, Anurag and Dehghani, Mostafa and Shen, Zhuoran and Wang, Xiao and Zhai, Xiaohua and Kipf, Thomas and Houlsby, Neil},
title = {{Simple Open-Vocabulary Object Detection}},
year = {2022},
isbn = {978-3-031-20079-3},
publisher = {},
address = {},
doi = {10.1007/978-3-031-20080-9_42},
booktitle = {Computer Vision – ECCV 2022},
pages = {},
numpages = {},
location = {Tel Aviv, Israel}
}

@InProceedings{liu2024groundingdino,
author="Liu, Shilong
and Zeng, Zhaoyang
and Ren, Tianhe
and Li, Feng
and Zhang, Hao
and Yang, Jie
and Jiang, Qing
and Li, Chunyuan
and Yang, Jianwei
and Su, Hang
and Zhu, Jun
and Zhang, Lei",
title="{Grounding DINO: Marrying DINO with Grounded Pre-training for Open-Set Object Detection}",
booktitle="Computer Vision -- ECCV 2024",
year="2025",
isbn="978-3-031-72970-6"
}

@article{
oquab2024dinov2,
title={{DINOv2: Learning Robust Visual Features without Supervision}},
author={Maxime Oquab and Timoth{\'e}e Darcet and Th{\'e}o Moutakanni and Huy V. Vo and Marc Szafraniec and Vasil Khalidov and Pierre Fernandez and Daniel HAZIZA and Francisco Massa and Alaaeldin El-Nouby and Mido Assran and Nicolas Ballas and Wojciech Galuba and Russell Howes and Po-Yao Huang and Shang-Wen Li and Ishan Misra and Michael Rabbat and Vasu Sharma and Gabriel Synnaeve and Hu Xu and Herve Jegou and Julien Mairal and Patrick Labatut and Armand Joulin and Piotr Bojanowski},
journal={Transactions on Machine Learning Research},
issn={2835-8856},
year={2024},
note={Featured Certification}
}

@inproceedings{vinyals2016matching,
 author = {Vinyals, Oriol and Blundell, Charles and Lillicrap, Timothy and kavukcuoglu, koray and Wierstra, Daan},
 booktitle = {Advances in Neural Information Processing Systems},
 editor = {},
 pages = {},
 publisher = {Curran Associates, Inc.},
 title = {{Matching Networks for One Shot Learning}},
 volume = {},
 year = {2016}
}

@inproceedings{snell2017prototypical,
	title = {Prototypical {Networks} for {Few}-shot {Learning}},
	volume = {30},
	urldate = {2026-08-17},
	booktitle = {Advances in {Neural} {Information} {Processing} {Systems}},
	publisher = {Curran Associates, Inc.},
	author = {Snell, Jake and Swersky, Kevin and Zemel, Richard},
	year = {2017},
}

@inproceedings{finn2017maml,
author = {Finn, Chelsea and Abbeel, Pieter and Levine, Sergey},
title = {Model-agnostic meta-learning for fast adaptation of deep networks},
year = {2017},
booktitle = {Proceedings of the 34th International Conference on Machine Learning - Volume 70},
}

@inproceedings{
chen2019closerlook,
title={{A Closer Look at Few-shot Classification}},
author={Wei-Yu Chen and Yen-Cheng Liu and Zsolt Kira and Yu-Chiang Frank Wang and Jia-Bin Huang},
booktitle={International Conference on Learning Representations},
year={2019},
}

@misc{tian2020rethinking,
title={Rethinking Few-Shot Image Classification: a Good Embedding Is All You Need?},
ISBN={9783030586102},
abstractNote={The focus of recent meta-learning research has been onthe  development  of  learning  algorithms  that  can  quicklyadapt to test time tasks with limited data and low compu-tational cost.   Few-shot learning is widely used as one ofthe standard benchmarks in meta-learning. In this work, weshow that a simple baseline: learning a supervised or self-supervised representation on the meta-training set, followedby training a linear classifier on top of this representation,outperforms state-of-the-art few-shot learning methods. Anadditional boost can be achieved through the use of self-distillation.   This demonstrates that using a good learnedembedding model can be more effective than sophisticatedmeta-learning algorithms. We believe that our findings mo-tivate a rethinking of few-shot image classification bench-marks and the associated role of meta-learning algorithms.Code is available at:http://github.com/WangYueFt/rfs/.},
school={Springer},
author={Tian, Yonglong and Wang, Yue and Tenenbaum, Joshua B and Isola, Phillip John},
year={2020},
}

@InProceedings{wang2020tfa,
  title = 	 {{Frustratingly Simple Few-Shot Object Detection}},
  author =       {Wang, Xin and Huang, Thomas and Gonzalez, Joseph and Darrell, Trevor and Yu, Fisher},
  booktitle = 	 {Proceedings of the 37th International Conference on Machine Learning},
  year = 	 {2020},
  publisher =    {PMLR}
}

@INPROCEEDINGS{qiao2021defrcn,
  author={Qiao, Limeng and Zhao, Yuxuan and Li, Zhiyuan and Qiu, Xi and Wu, Jianan and Zhang, Chi},
  booktitle={2021 IEEE/CVF International Conference on Computer Vision (ICCV)}, 
  title={{DeFRCN: Decoupled Faster R-CNN for Few-Shot Object Detection}},
  doi={10.1109/ICCV48922.2021.00856}}

@InProceedings{Perez-Rua_2020_CVPR,
author = {Perez-Rua, Juan-Manuel and Zhu, Xiatian and Hospedales, Timothy and Xiang, Tao},
title = {{Incremental Few-Shot Object Detection}},
booktitle = {Proceedings of the IEEE/CVF Conference on Computer Vision and Pattern Recognition (CVPR)},
year = {2020}
}

@InProceedings{Lin_2024_CVPR,
    author    = {Lin, Jiehong and Liu, Lihua and Lu, Dekun and Jia, Kui},
    title     = {{SAM-6D: Segment Anything Model Meets Zero-Shot 6D Object Pose Estimation}},
    booktitle = {Proceedings of the IEEE/CVF Conference on Computer Vision and Pattern Recognition (CVPR)},
    month     = {June},
    year      = {2024},
    pages     = {27906-27916}
}

@InProceedings{Nguyen_2024_CVPR,
    author    = {Nguyen, Van Nguyen and Groueix, Thibault and Salzmann, Mathieu and Lepetit, Vincent},
    title     = {{GigaPose: Fast and Robust Novel Object Pose Estimation via One Correspondence}},
    booktitle = {Proceedings of the IEEE/CVF Conference on Computer Vision and Pattern Recognition (CVPR)},
    month     = {June},
    year      = {2024},
    pages     = {9903-9913}
}

@InProceedings{ornek2024foundpose,
author="{\"O}rnek, Evin P{\i}nar
and Labb{\'e}, Yann
and Tekin, Bugra
and Ma, Lingni
and Keskin, Cem
and Forster, Christian
and Hodan, Tomas",
title="{FoundPose: Unseen Object Pose Estimation with Foundation Features}",
booktitle="Computer Vision -- ECCV 2024",
year="2025",
isbn="978-3-031-73347-5"
}

@misc{sun2022onepose,
      title={OnePose: One-Shot Object Pose Estimation without CAD Models}, 
      author={Jiaming Sun and Zihao Wang and Siyu Zhang and Xingyi He and Hongcheng Zhao and Guofeng Zhang and Xiaowei Zhou},
      year={2022},
      eprint={2205.12257},
      archivePrefix={arXiv},
      primaryClass={cs.CV},
}

@inproceedings{liu2022gen6d,
  title={{Gen6D: Generalizable Model-Free 6-DoF Object Pose Estimation from RGB Images}},
  author={Liu, Yuan and Wen, Yilin and Peng, Sida and Lin, Cheng and Long, Xiaoxiao and Komura, Taku and Wang, Wenping},
  booktitle={ECCV},
  year={2022}
}

@inproceedings{Mildenhall_2020_ECCV_NeRF,
booktitle = {Proceedings of the European Conference on Computer Vision (ECCV)},
author = {Ben Mildenhall and Pratul P. Srinivasan and Matthew Tancik and Jonathan T. Barron and Ravi Ramamoorthi and Ren Ng},
title = {{NeRF: Representing Scenes as Neural Radiance Fields for View Synthesis}},
year = {2020},
entrytype = {inproceedings},
id = {mildenhall2020nerf}
}

@InProceedings{Gupta_2014_ECCV,
author="Gupta, Saurabh
and Girshick, Ross
and Arbel{\'a}ez, Pablo
and Malik, Jitendra",
title="{Learning Rich Features from {RGB-D} Images for Object Detection and Segmentation}",
booktitle="ECCV 2014",
}

@INPROCEEDINGS{Eitel_2015_IROS,
  author = {Andreas Eitel and Jost Tobias Springenberg and Luciano Spinello and Martin Riedmiller and Wolfram Burgard},
  title = {{Multimodal Deep Learning for Robust {RGB-D} Object Recognition}},
  booktitle = {IEEE/RSJ International Conference on Intelligent Robots and Systems (IROS)},
  year = 2015,
  address = {Hamburg, Germany},
}

@article{Bagherinezhad_2016_AAAI,
title = {{Are Elephants Bigger than Butterflies? Reasoning about Sizes of Objects}},
author = {Hessam Bagherinezhad and Hannaneh Hajishirzi and Yejin Choi and Ali Farhadi},
booktitle = {AAAI},
year = {2016},
}

@inproceedings{Chiatti_2021_AAAI_MAKE,
  author    = {Chiatti, Agnese and Motta, Enrico and Daga, Enrico and Bardaro, Gianluca},
  title     = {{Fit to Measure: Reasoning about Sizes for Robust Object Recognition}},
  booktitle = {Proceedings of the {AAAI} 2021},
  editor    = {Martin, Andreas and Hinkelmann, Knut and Fill, Hans-Georg and Gerber, Aurona and Lenat, Doug and Stolle, Reinhard and van Harmelen, Frank},
  year      = {2021},
}

@article{McNemar_1947,
title={{Note on the Sampling Error of the Difference Between Correlated Proportions or Percentages}},
DOI={10.1007/BF02295996},
journal={Psychometrika},
author={McNemar, Quinn},
year={1947},
}

@article{dietterich1998,
author = {Dietterich, Thomas G.},
title = {Approximate statistical tests for comparing supervised classification learning algorithms},
year = {1998},
publisher = {MIT Press},
address = {Cambridge, MA, USA},
issn = {0899-7667},
doi = {10.1162/089976698300017197},
journal = {Neural Comput.},
month = oct,
numpages = {29}
}

@INPROCEEDINGS{Janoch_2011_ICCV,
  author={Janoch, Allison and Karayev, Sergey and Yangqing Jia and Barron, Jonathan T. and Fritz, Mario and Saenko, Kate and Darrell, Trevor},
  booktitle={2011 IEEE International Conference on Computer Vision Workshops (ICCV Workshops)}, 
  title={{A category-level 3-D object dataset: Putting the Kinect to work}}, 
  year={2011},
  volume={},
  number={},
  pages={1168-1174},
  doi={10.1109/ICCVW.2011.6130382}}

@misc{Xiaofei_2021_Arxiv,
      title={{SizeNet: Object Recognition via Object Real Size-based Convolutional Networks}}, 
      author={Xiaofei Li and Zhong Dong},
      year={2021},
      eprint={2105.06188},
      archivePrefix={arXiv},
      primaryClass={cs.CV},
}

@article{baghooee2026insect,
title = {Insect size matters: Using image and dimensions together improves image classification},
journal = {Ecological Informatics},
volume = {94},
pages = {103664},
year = {2026},
issn = {1574-9541},
doi = {https://doi.org/10.1016/j.ecoinf.2026.103664},
author = {Melika Baghooee and Quentin Geissmann}
}

@misc{tao2026appearancefailsgeometryrecognizes,
      title={{Where Appearance Fails, Geometry Recognizes: A CAD-Free 3D Shape Prior That Complements Vision Foundation Models}}, 
      author={Chenxi Tao and Seung-Kyum Choi},
      year={2026},
      eprint={2609.04381},
      archivePrefix={arXiv},
      primaryClass={cs.CV},
}

\end{document}